\documentclass[11pt]{article}

\usepackage{fullpage}
\makeatother

\usepackage{amsmath,amsthm,amstext,amsfonts,amssymb,mathrsfs,euscript}
\usepackage{subfig,color,setspace,crop,sublabel}
\usepackage{epsfig,epsf,psfrag,amssymb,amsfonts,latexsym,graphicx,bm,cite,xcolor,url,array}
\usepackage{algpseudocode}
\usepackage{algorithm,caption}
\usepackage{color}

\newtheorem{proposition}{Proposition}
\newcommand{\st}{{s.t. }}
\usepackage{graphicx}
\usepackage{wrapfig}
\usepackage{hyperref}

\begin{document}
\title{
Statistical Structure Learning, Towards a Robust Smart Grid
}

\author{Hanie Sedghi and Edmond Jonckheere\\
Department of Electrical Engineering\\
University of Southern California\\
Los Angeles, California 90089-2563\\
Email: \{hsedghi,~jonckhee\}@usc.edu}

\date{}
\maketitle

\begin{abstract}
Robust control and maintenance of the grid relies on accurate data. 
Both PMUs and state estimators are prone to false data injection attacks. Thus, it is crucial to have a mechanism for fast and accurate detection of an agent maliciously tampering with the data---for both preventing attacks that may lead to blackouts, 
and for routine monitoring and control tasks of current and future grids. 
We propose a decentralized false data injection detection scheme based on Markov graph of the bus phase angles. 
We utilize the Conditional Covariance Test (CCT) to learn the structure of the grid.  
Using the DC power flow model, we show that under normal circumstances, and because of walk-summability of the grid graph, the Markov graph of the voltage angles can be determined by the power grid graph. 
Therefore, a discrepancy between calculated Markov graph and learned structure should trigger the alarm. 
Local grid topology is available online from the protection system and we exploit it to check for mismatch. 
Should a mismatch be detected, we use correlation anomaly score to detect the set of attacked nodes.
Our method can detect the most recent stealthy deception attack on the power grid that assumes knowledge of bus-branch model of the system and is capable of deceiving the state estimator, damaging power network observatory, control, monitoring, demand response and pricing schemes~\cite{kosut2010malicious}.
Specifically, under the stealthy deception attack, the Markov graph of phase angles changes. In addition to detect a state of attack, our method can detect the set of attacked nodes.
To the best of our knowledge, our remedy is the first to comprehensively detect this sophisticated attack and it does not need additional hardware. Moreover, our detection scheme is successful no matter the size of the attacked subset.
Simulation of various power networks confirms our claims.
%
\end{abstract}
%
%
\section{Introduction} 

Synchronous Phasor Measurement Units (PMUs) are being massively deployed throughout the grid 
and provide the dispatcher with time-stamped measurements relevant to the state of grid health as well as the required data for controlling and monitoring the system. Currently, PMU's provide the fastest measurements of grid status. 
As a result, recent monitoring and control schemes rely primarily on PMU measurements. 
For example,~\cite{diao} tries to increase voltage resilience to avoid voltage collapse by using synchronized PMU measurements and decision trees. 
In addition,
\cite{Giannakis,Wiesel,hezhangJ}~
rely on phase angle measurements for fault detection and localization.
Nevertheless, we need to consider that it is not economically feasible to place PMUs in every node. Therefore, in some nodes in the system, State Estimators will still be used. PMUs are prone to false data injection attack and even if we do not consider that, part of the grid using the state estimators is the back door to false data injection attacks. Therefore, aforementioned methods 
 can be deluded by false data injection attack. Thus, it is crucial to have a mechanism  for fast and accurate discovery of malicious tampering; both for preventing the attacks that may lead to blackouts, and for routine monitoring and control tasks of the smart grid, including state estimation and optimal power flow. It should be noted that our immunization scheme does not depend on use of PMU's in the network and it is applicable both for current and future grid.\\

\subsection{Summary of Results and Related Work}
We have designed a decentralized false data injection attack detection mechanism that utilizes bus phase angles Markov graph. We utilize Conditional Covariance Test (CCT)~\cite{AnimaW} to learn the structure of smart grid. 
We show that under normal circumstances, and because of the grid structure, the Markov graph of voltage angles can be determined by the power grid graph; Therefore, a discrepancy between calculated Markov graph and learned structure triggers the alarm.
\begin{wrapfigure}{r}{.3\textwidth}
  \centering
\includegraphics[width=2in]{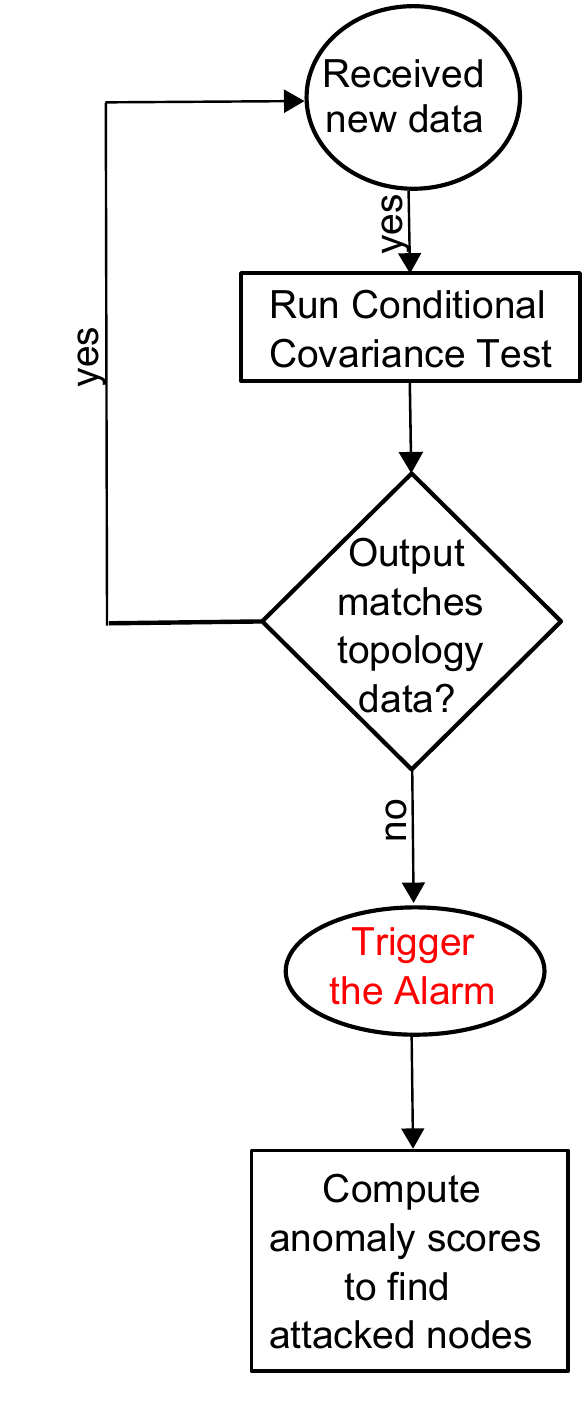}
\caption{\small Flowchart of our detection algorithm}
\label{fig:flowchart}
\end{wrapfigure}
Because of the connection between Markov graph of bus angle measurements and grid topology, our method can be implemented in a decentralized manner, i.e. at each sub-network. Currently, sub-network topology is available online and global network structure is available hourly~\cite{Giannakis}. 
Not only by decentralization can we increase the speed and get closer to online detection, 
but we also increase accuracy and stability by avoiding communication delays and synchronization problems when trying to send measurement data between locations far apart. 
We noticeably decrease the amount of exchanged data to address privacy concerns as much as possible.\\
We show that our method can detect the most recently designed attack on the power grid that can deceive the State Estimator~\cite{Teixreia2011}. The attacker is equipped with vital data and assumes the knowledge of bus-branch model of the grid. To the best of our knowledge, our method is the first to detect such a sophisticated attack comprehensively and efficiently with any number of attacked nodes. It should be noted that our method can detect that the system is under attack as well as the set of nodes under the attack. The flowchart is shown in Figure~\ref{fig:flowchart}.
%

Although the authors of~\cite{berkeley} suggest an algorithm for PMU placement such that this attack is observable, they only claim an algorithm for 2-node attack and empirical approaches for 3, 4, 5-node attacks. According to~\cite{berkeley}, for cases where more than two nodes are under attack the complexity of the approach is said to be \textit{``disheartening."} Considering the fact that finding the number of needed PMUs is NP hard and that~\cite{berkeley} gives an upper bound and use a heuristic method for PMU placement, we need to mention that our algorithm has no hardware requirements, the complexity does not depend on the number of nodes under attack and it works for any number of attacked nodes. It is worth mentioning that even in the original paper presenting the attack for a relatively small network (IEEE-30), seven measurements from five nodes are manipulated. So it seems that the 2-node attack is not the most probable one.\\
There has been another line of work towards computing the ``security index" for different nodes to find the set of nodes that are more vulnerable to false data injection attacks~\cite{Sandberg}. Although these attempts are appreciated, our method differs greatly from such perspectives as such methods do not detect the attack state when it happens and they cannot find the set of attack nodes.\\
Dependency graph approach is used in~\cite{hezhangJ} for topology fault detection in the grid. However, since attacks on state estimators are not considered, such methods can be deceived by false data injection. Furthermore,~\cite{hezhangJ} uses a constrained maximum likelihood optimization for finding the information matrix while here an advanced structure learning method is used that captures the power grid structure better. This is because in the power grid the edges are distributed over the network. This is discussed in section \ref{sec:CCT}.\\
In addition, we show that our method can detect the case where the attacker manipulates reactive power data to lead the state estimator to wrong estimates of voltage. Such an attack can be designed to fake a voltage collapse or tricking the operator to cause a voltage collapse. The detection can be done by linearisation of AC power flow and considering the fluctuations around steady state. Then following our algorithm, it readily follows that such an attack can also be detected following similar approach as we do here for bus phase angles and active power.\\
\paragraph{Paper Outline: } The paper is organized as follows. In section $2$,  we show that bus phase angles form a Gaussian Markov Random Field (GMRF) and discuss that their Markov graph can be determined by the grid structure.
In section $3$, we explain the Conditional Covariance Test (CCT)~\cite{AnimaW}, which we use for obtaining the Markov graph between bus phase angles, and discuss how we leverage it to perform optimally for the power grid.
The stealthy deception attack on the State Estimator is introduced in Section $4$. We elaborate on our detection scheme in Section $5$. Simulations are presented  in Section $6$. We elaborate on the reactive power counterpart in Section $7$ and Section $8$ concludes the paper.

\section{Preliminaries and problem formulation}
\subsection{Preliminaries}
A Gaussian Markov Random Field (GMRF) is a family of jointly Gaussian distributions, which factor
according to a given graph. Given a graph $G = (V,E)$, with $V = \{1, . . . , p\}$, consider
a vector of Gaussian random variables $X = [X_1,X_2, . . . ,X_p]^T$ , where each node $i\in V$ is
associated with a scalar Gaussian random variable $X_i$. A Gaussian Markov Random Field 
on $G$ has a probability density function (pdf) that can be parametrized as
\begin{align}
\label{user prob_distr}
f_X(x)\propto exp[-\frac{1}{2} x^TJx+h^Tx], 
\end{align}
where $J$ is a positive-definite symmetric matrix whose sparsity pattern corresponds to that
of the graph $G$. More precisely,
\begin{align*}
J(i,j)=0 \Longleftrightarrow (i,j)\notin E.
\end{align*}
The matrix $J = \Sigma  ^ {-1}$ is known as the \textit{potential} or \textit{information} matrix, the non-zero entries $J(i, j)$
as the edge potentials, and the vector $h$ as the vertex potential vector. In general, Graph $G=(V,E)$ is called the Markov graph (graphical model) underlying the joint probability distribution $f_X(x)$, 
where node set $V$ represents random variable set $\{X_i\}$ and the edge set $E$ is defined in order to satisfy local Markov property.
For a Markov Random Field, local Markov property states that  $X_i \perp {X}_{-\lbrace i,N(i)\rbrace}|X_{N(i)}$, 
where $X_{N(i)}$ represents all random variables associated with the neighbors of $i$ in graph $G$ 
and ${X}_{-\lbrace i,N(i)\rbrace}$ denotes all variables except for $X_i$ and $ X_{N(i)}$.

  




\subsection{Bus phase angles GMRF} \label{sec:Bmatrix}
We now apply the preceding to bus phase angles. 
The DC power flow model~\cite{state}
is often used for analysis of power systems in normal operations. When the system is stable, the phase angle differences are small, so $\sin(\theta_i-\theta_j) \sim \theta_i-\theta_j $. 
By the DC power flow model, system state $X$ can be described using bus phase angles. 
The active power flow on the transmission line connecting bus $i$ to bus $j$ is given by
\begin{align}
\label{one}
P_{ij} = b_{ij}(X_i - X_j),
\end{align}
where $X_i$ and $X_j$ denote the phasor angles at bus $i$ and
$j$ respectively, and $b_{ij}$ denotes the inverse of the line inductive reactance. 
The power injected at bus $i$ equals the algebraic sum of the powers flowing away from bus $i$:
\begin{align}
\label{two}
P_{i} =\sum_{j\neq i}P_{ij}=\sum_{j\neq i} b_{ij}(X_i - X_j).
\end{align}
When buses $i$ and $j$ are not connected, $b_{ij} = 0$. Thus, it follows that the phasor angle at bus $i$ could be represented as 
\begin{align}
\label{fund}
X_{i} =\sum_{j\neq i}\lbrace\frac{b_{ij}}{\sum_{i\neq j}b_{ij}}\rbrace X_j+\frac{1}{\sum_{j\neq i} b_{ij}} P_i.
\end{align}
Equation~\eqref{one} can also be rewritten in matrix form as
\begin{align}
\label{PBX}
P=BX,
\end{align}
where $P=[P_1,P_2,...,P_p]$ is the vector of injected active powers, $X=[X_1,X_2,...,X_p]$ is the vector of bus phase angles and 
\begin{align}
\label{Bdef}
B=\left\lbrace \begin{array}{lcl}
-b_{ij} & \mbox{if} & i\neq j, \\
\sum_{j\neq i} b_{ij} & \mbox{if} &i=j.
\end{array}
\right.
\end{align}
\paragraph{Remark:} Note that, because of linearity of the DC power flow model, the above equations are valid for both the phase angles $X$ and powers injected $P$ and for {\it fluctuations} of phase angles $X$ and powers injected $P$ around their steady-state values. Specifically, if we let $\widetilde{P}$ refer to the vector of active power fluctuations and $\widetilde{X}$ to the vector of phase angles fluctuations, we have $\widetilde{P}=B\widetilde{X}$. In the sequel, the focus is on the DC power flow equations.  Nevertheless, our analysis remains valid if we consider {\it fluctuations} around the steady-state values.

Because of load uncertainty, an injected power can be modeled as a random variable~\cite{Luettgen1993} and since an injected power models the superposition of many independent factors (e.g., loads, fluctuations of power outputs of wind turbines, etc.), it can be modeled as a Gaussian random variable. Thus, the linear relationship in \eqref{PBX} implies  that the difference of phasor angles across
a line could be approximated by a Gaussian random variable truncated within $[0, 2\pi)$. Considering the fixed phasor at the slack bus, 
it can be assumed that under steady-state conditions, phasor angle measurements are Gaussian random variables~\cite{hezhangJ}.

The next step is to find whether $X_i$'s satisfy local Markov property and, in the affirmative, discover the neighbor sets corresponding to each node. We do this by analyzing Equation~\eqref{fund}. If there were only the first term, we would conclude that set of nodes electrically connected to node $i$ satisfies the local Markov property, but the second term makes a difference. 
Below we argue that an analysis of the second term of \eqref{fund} shows that this term causes some second-neighbors of $X_i$ to have a nonzero term in the $J$ matrix. In addition, for nodes that are more than two hop distance apart, $J_{ij}=0$. Therefore, as opposed to the claim in~\cite{hezhangJ}, a second-neighbor relationship {\it does exist} in the $J$-matrix. \\
As stated earlier, power injection at different buses have Gaussian distribution. We can assume that they are independent and without loss of generality they are zero mean. Therefore, the probability distribution function for $P$ vector is
\begin{align*}
f_P(P) \propto e^{-\frac{1}{2}P^T P}
\end{align*}
Since $P=BX$, we have 
\begin{align*}
f_X(X) \propto e^{-\frac{1}{2}X^T B^T B X}
\end{align*}
Recalling the definition of probability distribution function for jointly Gaussian r.v. s in~\eqref{user prob_distr}, we get $J=B^T B$. Let $d(i,j)$ represent the hop distance between nodes $i$, $j$. Obviously, by definition of $B$ matrix, this leads to some nonzero $J_{ij}$ entries for $d(i,j)=2$. In addition, we state that
\begin{proposition}
Let $d(i,j)$ denote the hop distance between nodes $i$, $j$ on the power grid graph $G$. 
\\Assume that the fluctuations of the powers injected at the nodes are Gaussian and mutually independent. Then 
\begin{align*}
J_{ij}=0, \quad \quad \forall ~ d(i,j) >2.
\end{align*}
\end{proposition}
\begin{proof}
We argue by contradiction. 
Assume $J_{ij} \neq 0$ for some $d(i,j) >2$. Since $J_{ij}=\sum_k B_{ik}B_{jk}$, it follows that 
$\exists ~ k ~\st ~B_{ik}\neq 0, B_{jk} \neq 0$. By~\eqref{Bdef}, $B_{ik} \ne 0$ implies $d(i,k)=1$. 
From there on,  the triangle inequality implies that 
$d(i,j) \leq d(i,k)+d(k,j)=1+1=2$,   
which contradicts the assumption $d(i,j) >2$.
\end{proof}
It was shown in~\cite{sedghi:bookchapter} that for some graphs, the second-neighbor terms are smaller than the terms corresponding to the immediate electrical neighbors of $X_i$. More precisely, it was shown that for lattice structured grids, 
this approximation falls under the generic fact of the tapering off of Fourier coefficients~\cite{sedghi:bookchapter}. 
Therefore, we can approximate each neighborhood to the immediate electrical neighbors only. 
We can also proceed with the exact relationships. 
For simplicity, we opt the first-neighbor analysis. We explain shortly why CCT best describes this approximation.\\
Note that our detection method relies on graphical model of the variables. It is based on the fact that the Markov graph of bus phase angles changes under an attack. CCT is tuned with correct data and we prove that in case of the attack, Markov graph of compromised data does not follow the Markov graph of correct data. Hence, we can tune CCT by either exact relationships or approximate Markov graph. In both cases, the output in case of attack is different from the output tuned with correct data.
Therefore, it works for both approximate and exact neighborhoods. 

\section{Structure Learning}
In the context of graphical models, model selection means finding the exact underlying Markov graph among a group of random variables based on samples of those random variables.
There are two main class of methods for learning the structure of the underlying graphical model, convex and non-convex methods. $\ell_1$-regularized maximum likelihood estimators are the main class of convex methods ~\cite{Friedman&etal:07,Ravikumar&etal:08Arxiv,JanzaminAnandkumar:CovDecomp2012ArXiv}. In these methods, the inverse covariance matrix is penalized with a convex $\ell_1$-regularizer in order to encourage sparsity in the estimated Markov graph structure. Other types of methods are the non-convex or greedy methods ~\cite{AnimaW}.
As we are faced with GMRF in our problem, it would be useful to exploit one of these structure learning methods.
\subsection{Conditional Covariance Test} \label{sec:CCT}

In order to learn the structure of the power grid, we utilize the new Gaussian Graphical Model Selection method called {\it Conditional Covariance Test (CCT)}~\cite{AnimaW}. CCT method estimates the structure of underlying graphical model given i.i.d. samples of the random variables.
CCT method is shown in Algorithm~\ref{CCT}. 

\begin{algorithm}[t]
\caption{$CCT(x^n; \xi_{n,p},\eta)$ for structure learning using samples $x^n$~\cite{AnimaW}} 
\label{CCT}
\begin{algorithmic}
\State \textbf{Initialize} $\widehat{G}^n_p=(V,\emptyset)$\\
\State For each $(i,j) \in V^2$,
\If
{~~~~~~~~~$\min_{\substack{{S \subset V\setminus \{i,j\}}\\{|S| \leq \eta}}} \widehat{\Sigma}(i, j|S) > \xi_{n,p}$,\\}
\State  add $(i,j)$ to the edge set of $\widehat{G}^n_p$.\\
\EndIf
\State \textbf{Output}:$\widehat{G}^n_p$
\end{algorithmic}
\end{algorithm}
In Algorithm~\ref{CCT}, the output is an edge set corresponding to graph $G$ given                   
$n$ i.i.d. samples $x^n$, each of which has $p$ variables, a threshold $\xi_{n,p}$ (that depends on both $p$ and $n$) and a constant $\eta \in \mathbb{N}$, which is related to the local vertex separation property (described later). In our case, each bus phase angle represents one of the $p$ variables.\\
The sufficient condition for output of CCT to have structural consistency with the underlying Markov graph between variables is that the graph has to satisfy
local separation property and  walk-summability~\cite{AnimaW}. An ensemble of graphs has the $(\eta,\gamma)$-local separation property if 
 for any  $(i,j) \notin E(G)$, the maximum number of paths between $i$,$j$ of length at most $\gamma$ does not exceed $\eta$. A Gaussian model is said to be $\alpha$-walk summable
if 
$||\bar{\textbf{R}}|| \leq \alpha < 1$
where $\bar{\textbf{R}}=[|r_{ij}|]$ and $||.||$ denotes the spectral or 2-norm of matrix, which for symmetric matrices is given by the maximum absolute eigenvalue~\cite{AnimaW}. $ \textbf{R} $ is the matrix consisting of partial correlation coefficients. It is zero on diagonal entries and for non-diagonal entries we have 
\begin{align}
\label{partialr}
r_{ij}\triangleq & \frac{\Sigma(i,j|V \setminus{ \lbrace i,j \rbrace })}{\sqrt{\Sigma(i,i|{V\setminus { \lbrace i,j \rbrace}})\Sigma(j,j|{V \setminus { \lbrace i,j \rbrace}})}}\nonumber \\
=&-\frac{J(i,j)}{\sqrt{J(i,i)J(j,j)}}.
\end{align}  
$r_{ij}$, the \textit{partial correlation coefficient} between variables $X_i$ and $X_j$ for $i \neq j$, measures their conditional covariance given all other variables~\cite{Lauritzen:book}.\\
Whether we use the exact or approximate neighborhood relationship, the Markov graph of bus phase angles is an example of bounded local path graphs that satisfy local separation property. We also checked the analyzed networks for walk-summability condition. As shown in \eqref{partialr} and definition of $\alpha$-walk summablity, this property depends only on $J$ matrix and thus on the topology of the grid. $\alpha$-walk summablity does not depend on operating point of the grid. \\
It is shown in~\cite{AnimaW} that under walk summability, the effect of faraway nodes on covariance decays exponentially with distance and the error in approximating the covariance by local neighboring decays exponentially with distance. So by correct tuning of the threshold $\xi_{n,p}$ and with enough samples, we expect the output of CCT method to follow the grid structure.\\
It is worth mentioning that when we use CCT method for structure learning of phasor data, our method is robust against measurement noise. The reason is that CCT analyzes conditional covariance of its input data. Measurement noise is white noise and in addition uncorrelated from the data. As a result, it does not change the conditional correlation between phasor data. Thus, our method is immune to measurement noise.\\
CCT distributes the edges fairly uniformly across the nodes, while the $\ell_1$ method tends to cluster all the edges together between the ``dominant'' variables leading to a densely connected component and several isolated points~\cite{AnimaW}. Therefore, CCT is more suitable for detecting the structure of the power grid where the edges are distributed over the network. It should be noted that the computational complexity of CCT is $O(p^{\eta+2})$, which is efficient for small $\eta$~\cite{AnimaW}. $\eta$ is the parameter associated with local separation property described above. \\
The sample complexity associated with CCT method is $n=\Omega(J_{min}^{-2}\log{p})$, where $J_{min}$ is the minimum absolute edge potential in the model~\cite{AnimaW}.

\subsection{Decentralization} 
We want to find the Markov graph of our bus phasor measurements. Since we have made the connection between electrical connectivity and correlation, this helps us to decentralize our method to a great extent. 
We consider the power network in its normal operating condition. 
It consists of different areas connected together via border nodes. 
Therefore, we decompose our network into these sub-areas. 
Our method can be performed locally in the sub-networks. 
The sub-network connection graph is available online from the protection system at each sub-network 
and can be readily compared with the bus phase angle Markov graph. 
In addition, only for border nodes we need to consider their out-of-area neighbors as well. 
This can be done either by solving the power flow equations for that border link 
or by receiving measurements from neighbor sub-networks.
Therefore, we run CCT for each sub-graph to figure out its Markov graph. Then we compare it with online network graph information to detect false data injection attack.\\
This decentralization reduces complexity and increases speed. Our decentralized method is a substitute for considering all measurements throughout the power grid, which requires a huge amount of data exchange and computation. In addition to having less nodes to analyze, this decentralization leads us to a smaller $\eta$ and greatly reduces computational complexity, which makes our method capable of being executed in very large  networks.
Furthermore, since structure learning is performed locally, non-linear faraway relationships intrinsic to power systems do not play a role and our assumptions remain valid.
Moreover, utility companies are not willing to expose their information for economical competition reasons and there have been several attempts to make them do that~\cite{Sankar}. Thus it is desired to reduce the amount of data exchange between different areas and our method adequately fulfills this requirement.\\

%
\subsection{Online calculations}
For fast monitoring the power grid, we need an on-line algorithm. 
As we show in this section, our algorithm can be developed as an iterative method that processes new data without the need for reprocessing earlier data.
Here, we derive an iterative formulation for the sample covariance matrix. 
Then we use it to calculate the conditional covariance using
\begin{align*}
\widehat{\Sigma}(i, j|S) := \widehat{\Sigma}(i, j)- \widehat{\Sigma}(i, S)\widehat{\Sigma}^{-1}
(S, S)\widehat{\Sigma}(S, j).
\end{align*}
 As we know, in general
\begin{align*}
\Sigma= E[(X-\mu)(X-\mu)^T]=E[XX^T]-\mu\mu^T.
\end{align*}
 Let $\widehat{\Sigma}^{(n)}(X)$ denote the sample covariance matrix for a vector $X$ of $p$ elements from $n$ samples and let $\widehat{\mu}^{(n)}(X)$ be the corresponding sample mean. 
In addition, let $X^{(i)}$ be the $i$th sample of our vector. Then we have
 \begin{align}
 \label{iter}
\widehat{\Sigma}^{(n)}(X)=\frac{1}{n-1}\sum_{i=1}^n X^{(i)}{X^{(i)}}^T-\widehat{\mu}^{(n)}{\widehat{\mu}^{(n)^T}}.
 \end{align}
Therefore,
 \begin{align}
 \label{update}
\widehat{\Sigma}^{(n+1)}(X)=\frac{1}{n}[\sum_{i=1}^n X^{(i)}{X^{(i)}}^T+X^{(n+1)}{X^{(n+1)}}^T] -\widehat{\mu}^{(n+1)}{\widehat{\mu}^{(n+1)^T}}, 
 \end{align}
 where
 \begin{align}
\label{samplemean}
{\widehat{\mu}^{(n+1)}} =\frac{1}{n+1}[n{\widehat{\mu}^{(n)}}+X^{(n+1)}].
 \end{align}
By keeping the first term in~\eqref{iter} and the sample mean~\eqref{samplemean}, our updating rule is~\eqref{update}. Thus, we revise the sample covariance as soon as any bus phasor measurements changes and leverage it to reach conditional covariances needed for CCT.
It goes without saying that if the system demand and structure does not change and the system is not subject to false data injection attack, the voltage angles at nodes remain the same and there is no need to run any  algorithm.
\section{Stealthy Deception Attack}
The most recent and most dreaded false data injection attack on the power grid was introduced in~\cite{Teixreia2011}. It assumes knowledge of bus-branch model of the system and is capable of deceiving the state estimator, damaging power network observatory, control, monitoring, demand response and pricing schemes~\cite{kosut2010malicious}.

For a $p$-bus electric power network, the $l=2p-1$ dimensional 
state vector $x$ is $(\theta^T,V^T)^T$, where  $V=(V_1,...,V_p)$ is the vector of voltage bus magnitudes and $\theta=(\theta_2,...,\theta_p)$ the vector of phase angles. It is assumed that the nonlinear measurement model for state estimation is defined by $z=h(x)+\epsilon$,
where $h(.)$ is the measurement function, $z=(z_P,z_Q)$ is the measurement vector consisting of active and reactive power flow measurements and $\epsilon$ is the measurement error. $H(x^k):=\frac{dh(x)}{dx}|_{x=x^k}$ denotes the Jacobian matrix of the measurement model $h(x)$ at $x^k$.\\
The goal of the stealthy deception attacker is to compromise the measurements available to the State Estimator (SE) such that
$z^a=z+a$,
where $z^a$ is the corrupted measurement and $a$ is the attack vector. Vector $a$ is designed such that the SE algorithm converges and the attack $a$ is undetected by the Bad Data Detection scheme. 
Then it is shown that, under the DC power flow model,
such an attack can only be performed locally with $a \in \mathrm{Im}(H)$,  
where $H=H_{P\theta}$ is the matrix connecting the vector of bus injected powers to the vector of bus phase angles, i.e., $P=H_{P\theta} \theta$. The attack is shown in Figure\ref{fig:SE}.
\begin{figure}[t]
\centering
\captionsetup{type=figure}
\includegraphics[width=3.5in]{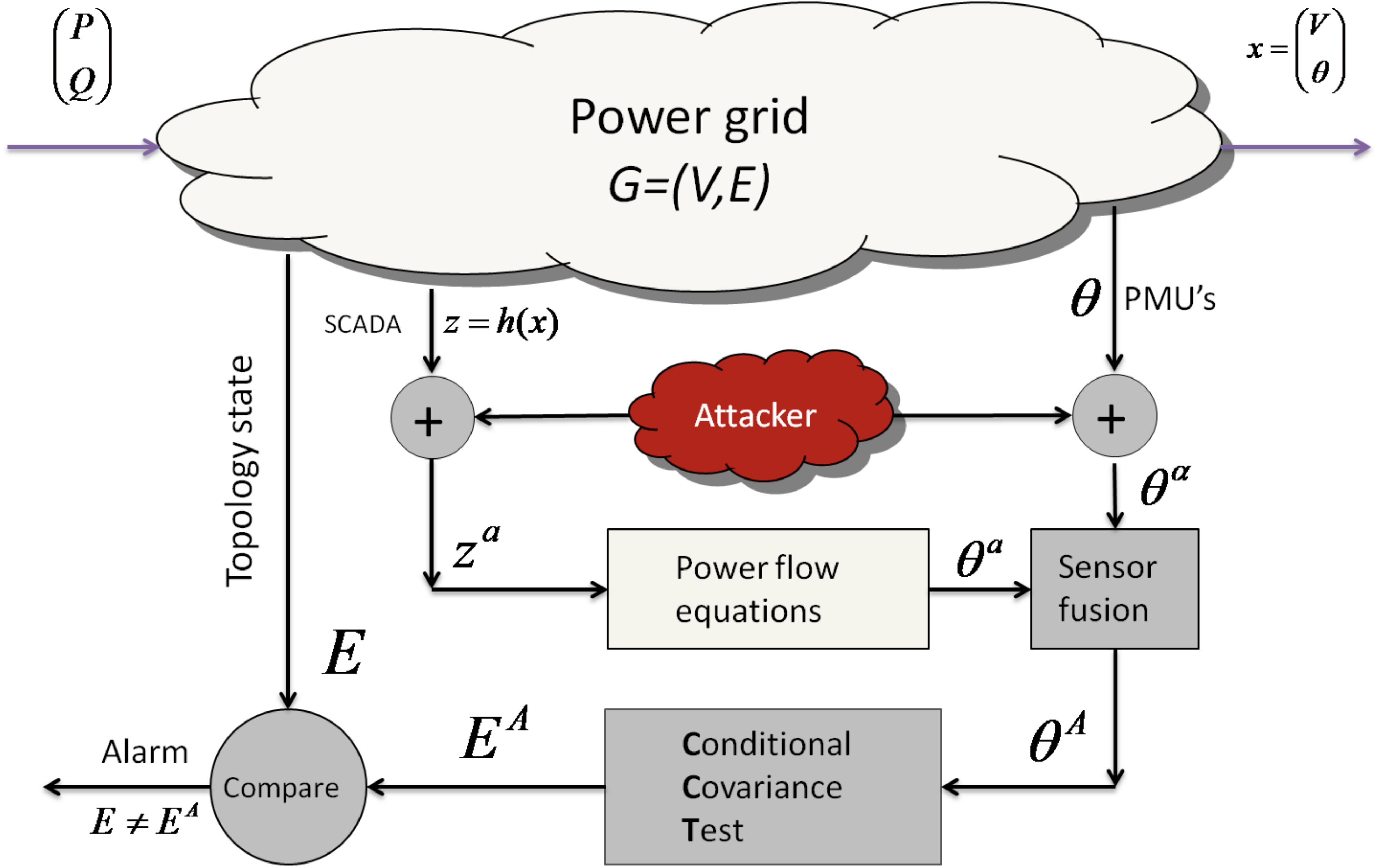}
\caption{\small Power grid under a cyber attack}
\label{fig:SE}
\end{figure}
%
%
\section{Stealthy Deception Attack Detection}
As mentioned earlier, the fundamental idea behind our detection scheme is that of structure learning. 
Our learner, the CCT method, is first tuned with correct data representing the data structure, 
which corresponds to the grid graph. Therefore, any attack that changes the structure alters the output of CCT method and this triggers the alarm.
Let us consider the aforementioned attack more specifically. As we are considering the DC power flow model 
and all voltage magnitudes are considered to be 1 p.u., 
the state vector introduced in~\cite{Teixreia2011} reduces to the vector of voltage angles, $X$. Since 
$a \in \mathrm{Im}(H)$, $\exists d$ such that $a=Hd$ and
\begin{align*}
z^a=z+a=H(X+d)=HX^a,
\end{align*}
where $X^a$ represents the vector of angles when the system is under attack,
$z^a$ is the attacked measurement vector and $X$ is the actual phasor angle vector. Considering~\eqref{two}, we have $H_{ij}=-b_{ij}$ for $i \neq j$ and $H_{ii}=\sum_{i\neq j}b_{ij}$, where $b_{ij}$ denotes the inverse of the line inductive reactance. We have
\begin{align}
\label{Xa}
X^a=X+d=H^{-1}P+H^{-1}a=H^{-1}(P+a).
\end{align}
As the definition of $H$ matrix shows, it is of rank $p-1$. Therefore, the above $H^{-1}$ denotes the pseudo inverse of $H$ matrix. Another way to address this singularity is to remove the row and column associated with slack bus. \\
From~\eqref{Xa},
\begin{align*}
\Sigma{(X^a,X^a)}= H^{-1}[\Sigma(P+a,P+a)]{H^{-1}}^T
= H^{-1}[\Sigma(P,P)+\Sigma(a,a)]{H^{-1}}^T.
\end{align*}
The above calculation assumes the attack vector being independent of current values in the network, 
as demonstrated in definition of the attack~\cite{Teixreia2011}. \\
An attack is considered successful if it causes the operator to  make a wrong decision. For that matter, the attacker would not insert just one wrong sample. In addition, if the attack vector remains constant, it does not cause any reaction. This eliminates the cases of constant attack vectors. Therefore, the attacker is expected to insert random vectors $a$ during some samples. Thus $ \Sigma(a,a) \neq 0 $ and
\begin{align}
\label{attack_d}
\Sigma{(X^a,X^a)} \neq \Sigma(X,X).
\end{align} 
It is not difficult to show that if we remove the assumption on independence of attack vector and injected power, \eqref{attack_d} still holds.\\
Considering~\eqref{attack_d} and the fact that matrix inversion enjoys uniqueness property, this means that in case of an attack, the new $\Sigma^{-1}$ will not be the same as network's $J$ matrix in normal condition, i.e.,
$\Sigma^{-1}{(X^a,X^a)} \neq J_{normal}$, 
and as a result,
the output of CCT method will not follow the grid structure.

We use this mismatch to trigger the alarm. It should be noted that acceptable load changes do not change the Markov graph and as a result do not lead to false alarms. The reason is that such changes do not falsify DC power flow assumption and the Markov graph will continue to follow the defined information matrix.\\
After the alarm is triggered, the next step is to find which nodes are under attack.

\subsection{Detecting the Set of Attacked Nodes}
We use the \textit{correlation anomaly} metric~\cite{anomaly} to find the attacked nodes. This metric quantifies the contribution of each random variable to the difference between probability densities in two cases while considering the sparsity of the structure. Kullback-Leibler (KL) divergence is used as the measure of difference. As soon as an attack is detected, we use the attacked information matrix and the information matrix corresponding to the current topology of the grid to compute the anomaly score for each node. The nodes with highest anomaly score are announced as the nodes under attack. We investigate the implementation details in the next section. 

It should be noted that the attack is performed locally and because of local Markov property, we are certain that no nodes from other sub-graphs contributes to the attack.

We should emphasize that the considered attack assumes the knowledge of the system's bus-branch model. 
Therefore, the attacker is equipped with very critical information. Yet, we can mitigate such an ''intelligent'' attack.
%
\subsection{Reactive power versus voltage amplitude}  \label{sec:AC}
As mentioned before, with similar calculations, we can consider the case where the attacker manipulates reactive power data to lead the state estimator to wrong estimates of voltage. Such an attack can be designed to fake a voltage collapse or tricking the operator to cause a change in the normal state of the grid. For example, if the attacker fakes a decreasing trend in voltage magnitude of a part of the grid, the operator will send more reactive power to that part and thus this causes voltage overload/underload. At this point, the protection system disconnects the corresponding lines. This can lead to outage in some areas and in worse cases to overloading in other parts of the grid that might cause blackouts and cascading events. The detection can be done by linearization of AC power flow and considering the fluctuations around steady state. Then following our algorithm, it readily follows that such an attack can also be detected following a similar approach as we developed here for bus phase angles and active power. 

In this section we show how this analogy can be established.
The AC power flow states that the real power and the reactive power flowing from bus $i$ to bus $j$ are, respectively,  
\begin{align*}
&P_{ij} =G_{ij}V^2_i -G_{ij}V_iV_j\cos (\theta _i - \theta _j) + b_{ij}V_iV_j\sin( \theta _i - \theta _j ),\\
&Q_{ij} = b_{ij}V^2_i -b_{ij}V_iV_j\cos(\theta _i -\theta _j) -G_{ij}V_iV_j \sin(\theta _i-\theta _j),
\end{align*}
where $V_i$ and $\theta_i$ are the voltage magnitude and phase angle, resp., at bus \#i and 
$G_{ij}$ and $b_{ij}$ are the conductance and susceptance, resp., of line $ij$. 
From ~\cite{Reza2010}, we obtain the following approximation of the AC {\it fluctuating} power flow:
\begin{align*}
&\widetilde{P}_{ij}=(b_{ij}\overline{V}_i \overline{V}_j\cos \overline{\theta}_{ij})(\widetilde{\theta}_i-\widetilde{\theta}_i),\\
&\widetilde{Q}_{ij}=(2b_{ij}\overline{V}_i -b_{ij}\overline{V}_j\cos \overline{\theta}_{ij})\tilde{V}_i-(b_{ij}\overline{V}_i\cos\overline{\theta}_{ij})\widetilde {V}_j,
\end{align*}
where bar denotes steady-state value, tilde means fluctuation around the steady-state value, and $\overline{\theta}_{ij}=\overline{\theta}_{i}-\overline{\theta}_{j}$. These fluctuating values 
due to renewables and variable loads justify the utilization of probabilistic methods in power grid problems.\\
Now assuming that for the steady-state values of voltages we have $\overline{V}_i=\overline{V}_j \simeq 1 p.u.$ 
(per unit), and the fluctuations in angles are about the same such that $\cos \theta_{ij}=1$, we have
\begin{align}
\sublabon{equation}
\label{P}
\widetilde{P}_{ij}=b_{ij}(\widetilde{\theta}_i-\widetilde{\theta}_j),\\
\label{Q}
\widetilde{Q}_{ij}=b_{ij}(\widetilde{V}_i-\widetilde{V}_j).
\end{align}
\sublaboff{equation}
It is clear from~\eqref{P}-\eqref{Q} that we can follow the same discussions we had about real power and voltage angles, with reactive power and voltage magnitudes.\\
It can be argued that, as a result of uncertainty, the aggregate reactive power at each bus can be approximated as a Gaussian random variable and, because of Equation \eqref{Q}, 
voltage fluctuations around the steady-state value can be approximated as Gaussian random variables. 
Therefore, the same path of approach as for phase angles can be followed to show the GMRF property for voltage amplitudes. 
Comparing \eqref{Q} with \eqref{one} makes it clear that the same matrix, i.e., 
the $B$ matrix developed in Section~\ref{sec:Bmatrix}, 
is playing the role of correlating the voltage amplitudes; 
therefore, assuming that the statistics of the active and reactive power fluctuations are similar, 
the underlying graph is the same. This can be readily seen by comparing \eqref{P} and \eqref{Q}. 
\section{Simulation}
In this section, we present experimental results. We consider IEEE 14-bus system as well as IEEE-30 bus system. First, we feed the system with Gaussian demand and simulate the power grid. 
We use MATPOWER~\cite{MATPOWER} for solving the DC power flow equations for various demand and use the resulting angle measurements as the input to CCT algorithm. We leverage YALMIP~\cite{YALMIP} and SDPT3~\cite{sdpt3} to perform CCT method in MATLAB.

With the right choice of parameters and threshold, and enough measurements, the Markov graph follows the grid structure. We use the edit distance metric for tuning the threshold value. The edit distance between two graphs reveals the number of different edges in graphs, i.e., edges that exist in only one of the two graphs.

After the threshold is set, our detection algorithm works in the following manner.  Each time the procedure is initiated, i.e., any PMU angle measurement or state estimator output changes, it updates the conditional covariances based on new data, runs CCT and checks the edit distance between the Markov graph of phasor data and the grid structure. A discrepancy triggers the alarm. Subsequently, the system uses  metric to find all the buses under the attack. The flowchart of our method is shown in Figure~\ref{fig:flowchart}.

Next we introduce the stealthy deception attack on the system. The attack is designed according to the description above, i.e., it is a random vector such that $a \in \mathrm{Im}(H)$.
The attack is claimed to be successful only if performed locally on connected nodes. Having this constraint in mind, for IEEE-14 test case the maximum number of attacked nodes is six and for IEEE 30-bus system this number is eight.  For IEEE-14 network, we consider the cases where two to six nodes are under attack. For IEEE-30 network, we consider the cases where two to eight nodes are under attack. For each case and for each network, we simulate all possible attack combinations. This is to make sure we have checked our detection scheme against all possible stealthy deception attacks. Each case is repeated 1000 times for different attack vector values.

When the attacker starts tampering with the data, the corrupted samples are added to the sample bin of CCT and therefore they are used in calculating the sample covariance matrix. With enough number of corrupted samples, our algorithm can be arbitrarily close to $100\%$ successful in detecting all cases and types of attacks discussed above, for both IEEE-14 and IEEE-30 bus systems. 

The reason behind the trend shown in Figure 3 is that the Markov graph is changing from following the true information matrix towards showing some other random relationship that the attacker is producing. When the number of compromised  samples increases, they gain more weight in the sample covariance. Hence, the chance of change in Markov graph and detection increases.
 The minimum number of corrupted samples for having almost $100\%$ detection rate for IEEE 14-bus system is 130 and it is 50 for IEEE 30-bus system. Since IEEE-$30$ is more sparse compared to IEEE 14-bus system, our method performs better in the former case.Yet, 
for a 60 Hz system, detection speed for IEEE 14-bus system is quite amazing as well.

Another interesting fact is the detection rate trend as the number of corrupted measurements increases. This is shown in Figure~\ref{fig:drate} for IEEE 14-bus system. Detection rate is averaged over all possible attack scenarios. It can be seen that even for small number of corrupted measurements our method presents a good performance. The detection rate is $90\%$ with $30$ corrupted samples. 
\begin{figure}[t]
\centering
\captionsetup{type=figure}
\includegraphics[width=3.7in]{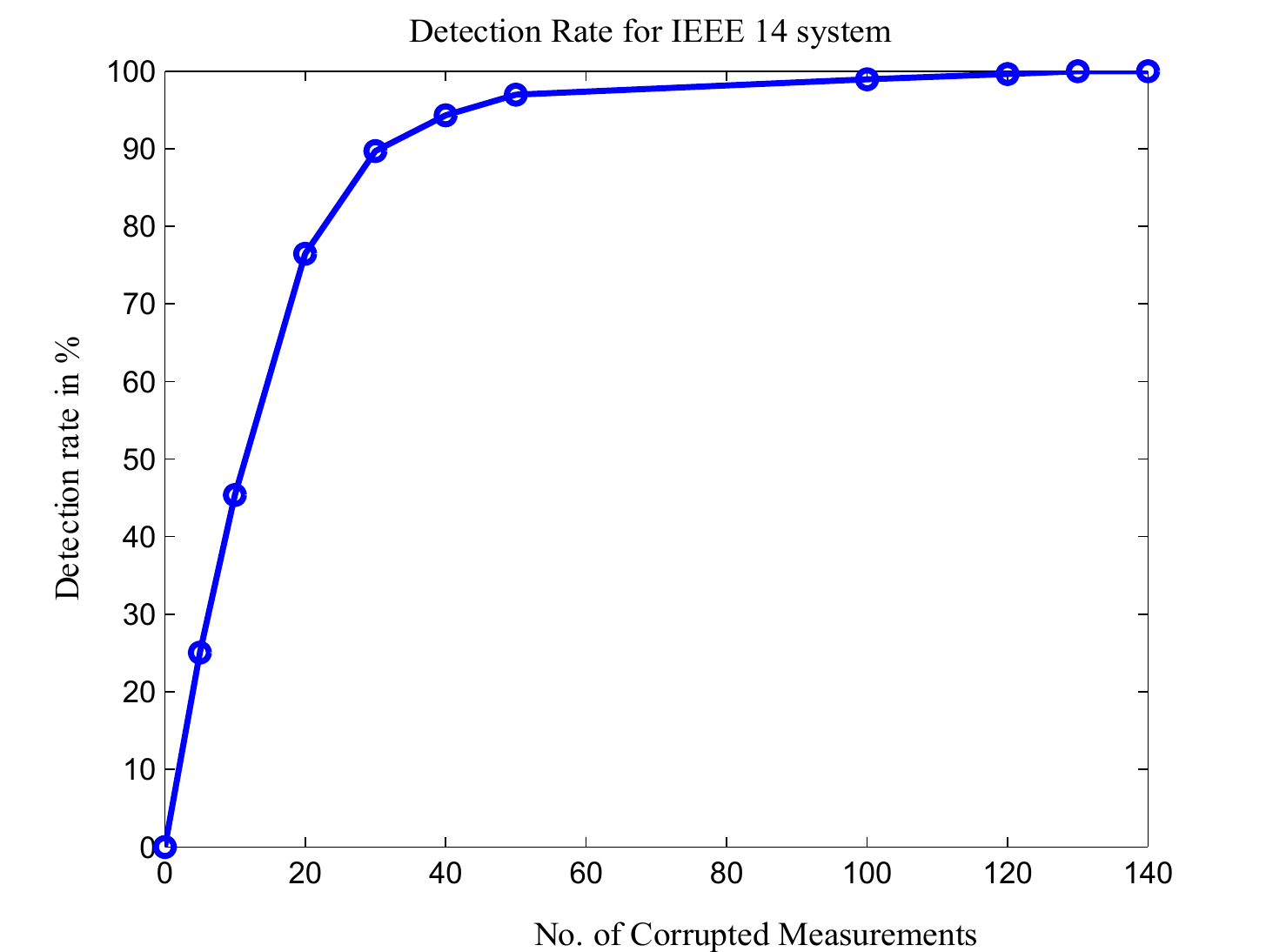}
\caption{\small Detection rate for IEEE 14-bus system}
\label{fig:drate}
\end{figure}

The next step is to find which nodes are attacked. 
As stated earlier, we use anomaly score metric~\cite{anomaly} to detect such nodes.
As an example, Figure~\ref{fig:point7} shows the anomaly score plot for 
the case where nodes $4, 5$ and $6$ are under the attack\footnote{The numbering system employed here is the one of the publised IEEE 14 system available at \url{https://www.ee.washington.edu/research/pstca/pf14/pg_tca14bus.htm}}. It means that a random vector is added to measurements at these nodes. This attack is repeated $1000$ times for different values building an attack size of $0.7$. Attack size refers to the expected value of the Euclidean norm of the attack vector.
\begin{figure}[b]
\centering
\captionsetup{type=figure}
\begin{psfrags}
\psfrag{Node Number}{Node Number}
\includegraphics[width=6in]{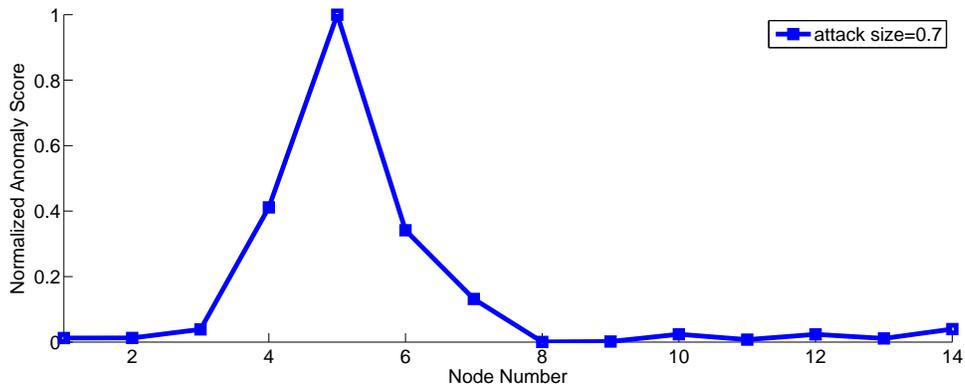}
\end{psfrags}
\caption{\small Anomaly score for IEEE 14-bus system, nodes 4, 5, 6 are under attack; Attack size is 0.7.}
\label{fig:point7}
\end{figure}
Simulation results show that as the attack size increases, the difference between anomaly score of the nodes under the attack and the uncompromised nodes increases and, as a result, it is easier to pinpoint the attacked nodes. For example, Figure~\ref{fig:3cases} compares the cases where the attack size is $1$, $0.7$ and $0.5$ for the previous attack scenario, i.e., nodes 4, 5, 6 are under attack.

It should be noted that in order for an attack to be successful in misleading the power system operator, the attack size should not be too small. More specifically, the attacker wants to make a change in the system state such that the change is noticeable, it results in wrong estimation in part of the grid or it triggers a reaction in the system. If the value of the system state under the attack is close to its real value, the system is not considered under the attack as it continues its normal operation.
It can be seen that, even for the smallest possible attack size 
that would normally not lead to operator to react, the anomaly score plot will remain reliable.
For example, in the considered attack scenario, the anomaly plot performs well even for an attack size of $0.3$, while it seems that a  potentially successful attack under normal standards needs a bigger attack size. 

\begin{figure}[t]
\begin{center}
\includegraphics[width=6in]{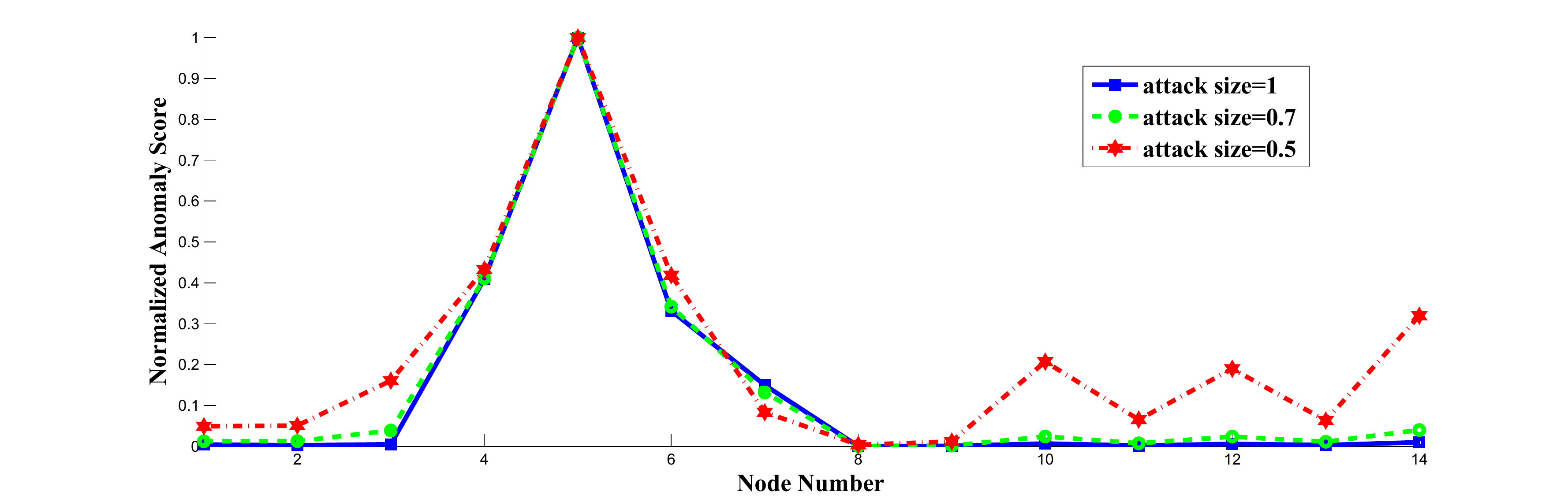}
\caption{\small Anomaly score for IEEE 14-bus system for different attack sizes. Nodes 4, 5, 6 are under attack. Attack sizes are 0.5, 0.7, 1.}
\label{fig:3cases}
\end{center}
\end{figure}
There is another important issue in setting the threshold for anomaly score. 
As can be seen in the above plots, a threshold of $0.3$ can give us a detection rate of $100\%$. If we increase the threshold, both detection rate and false alarm rate decrease, but for this application we want the design to satisfy the $100\%$ detection rate with a very low false alarm rate. It can readily be seen that our expectations are met. Simulation results for different attacks on IEEE 14-bus system shows that this threshold guarantees $100\%$ detection rate with a very low false alarm rate of $3.82\times 10^{-5}$. 

\section{Discussion and Conclusion}
We proposed a decentralized false data injection attack detection scheme that is capable of detecting the most recent stealthy deception attack on the smart grid SCADA. To the best of our knowledge, our remedy is the first to comprehensively detect this sophisticated attack. In addition to detecting the attack state, our algorithm is capable of pinpointing the set of attacked nodes. 

As stated earlier, the computational complexity of our method is polynomial and the decentralized property makes our scheme suitable for huge networks, yet with bearable complexity and run time. 
Consequently, our approach can be extended to bigger networks, namely IEEE-118 and IEEE-300 bus systems.

We also argued that our method is capable of detecting attacks that manipulate reactive power measurements to cause inaccurate voltage amplitude data. The latter attack scenario is also important as it can lead to, 
or mimic, a voltage collapse.

In conclusion, our method fortifies the power system against a large class of false data injection attacks 
and our technique could becomes essential for current and future grid reliability, security and stability.

We introduced change detection for graphical model of the system and showed that it can be leveraged to detect attacks on the system and controlling scheme. This fortifying scheme is very crucial to  maintain reliable control and stability. Change detection is a very well-known concept in control literature. Nevertheless, change detection in graphical models and using it for detecting the attacks is the new concept and our contribution to the field.

\subsection*{Acknowledgment}
We acknowledge detailed discussions with Majid Janzamin and thank him for valuable comments on graphical model selection and Conditional Covariance Test. The authors thank Reza Banirazi for discussions about power network operation and Phasor Measurement Units. We thank Anima Anandkumar for valuable insights into local separation property and walk summability.


\end{document}